\documentclass[12pt]{article}

\usepackage{amsmath,amsfonts}
\usepackage{algorithmic}
\usepackage{array}
\usepackage[caption=false,font=normalsize,labelfont=sf,textfont=sf]{subfig}
\usepackage{textcomp}
\usepackage{stfloats}
\usepackage{url}
\usepackage{verbatim}
\usepackage{graphicx}
\usepackage{amssymb}
\usepackage{amsthm}
\usepackage{inputenc}
\usepackage{balance}
\usepackage{xcolor}

\usepackage[linesnumbered, ruled]{algorithm2e}

\usepackage{caption}
\newtheorem{theorem}{Proposition}
\newtheorem{definition}{Definition}

\usepackage{graphicx}
\usepackage{url} 

\usepackage[style=authoryear,maxbibnames = 99,maxcitenames=3, backend=biber, giveninits, uniquename = false, uniquelist=false]{biblatex}
\DeclareFieldFormat{pages}{#1} 

\renewbibmacro{in:}{%
  \ifentrytype{article}
    {}
    {\bibstring{in}%
     \printunit{\intitlepunct}}}

\renewbibmacro*{journal+issuetitle}{%
  \usebibmacro{journal}%
  \setunit*{\addcomma\space}%
  \iffieldundef{series}
    {}
    {\newunit
     \printfield{series}%
     \setunit{\addspace}}%
  \usebibmacro{volume+number+eid}%
  \setunit{\addspace}%
  \usebibmacro{issue+date}%
  \setunit{\addcolon\space}%
  \usebibmacro{issue}%
  \newunit}

\DeclareNameAlias{sortname}{family-given}

\newcommand{\blind}{0}

\begin{document}

\def\spacingset#1{\renewcommand{\baselinestretch}%
{#1}\small\normalsize} \spacingset{1}


\if0\blind
{
  \title{\bf Unsupervised spectral-band feature identification for optimal process discrimination}  
  \author{Akash Tiwari
    and
    Satish Bukkapatnam \thanks{
    The authors gratefully acknowledge grant support from \textit{Lawrence Livermore National Laboratory}, \textit{National Science Foundation}, and \textit{Texas A\&M University X-grants }}\hspace{.2cm}\\
    \\
    Wm Michael Barnes '64 Department of Industrial and Systems Engineering, \\
    Texas A \& M University}
  \maketitle
  
} \fi

\if1\blind
{
  \bigskip
  \bigskip
  \bigskip
  \begin{center}
    {\LARGE\bf Unsupervised spectral-band identification for optimal process discrimination}
\end{center}
  \medskip
} \fi

\bigskip

\vspace{-0.5 in}
\begin{abstract}

Changes in real-world dynamic processes are often described in terms of differences in energies $\textbf{E}(\underline{\alpha})$ of a set of spectral-bands $\underline{\alpha}$. 
Given continuous spectra of two classes $A$ and $B$, or in general, two stochastic processes $S^{(A)}(f)$ and $S^{(B)}(f)$, $f \in \mathbb{R}^+$, we address the problem of identifying a subset of intervals of $f$ called spectral-bands $\underline{\alpha} \subset \mathbb{R}^+$ such that the energies $\textbf{E}(\underline{\alpha})$ of these bands can optimally discriminate between the two classes. 
This ubiquitous problem has received very little attention thus far.
We introduce EGO-MDA, an unsupervised method to identify optimal spectral-bands $\underline{\alpha}^*$ for given samples of spectra from two classes. 
EGO-MDA employs a statistical approach that iteratively minimizes an adjusted multinomial log-likelihood (deviance) criterion $\mathcal{D}(\underline{\alpha},\mathcal{M})$.
Here, Mixture Discriminant Analysis (MDA) aims to derive MLE of two GMM distribution parameters, i.e.,  $\mathcal{M}^* = \underset{\mathcal{M}}{\rm argmin}~\mathcal{D}(\underline{\alpha}, \mathcal{M})$ and identify a classifier that optimally discriminates between two classes for a given spectral representation.
The Efficient Global Optimization (EGO) finds the spectral-bands $\underline{\alpha}^* =  \underset{\underline{\alpha}}{\rm argmin}~\mathcal{D}(\underline{\alpha},\mathcal{M})$ for given GMM parameters $\mathcal{M}$.
While MDA has already existed for a few decades, pathological cases of low separation between mixtures and model misspecification on the classifier performance has been overlooked.
Under pathological cases, we discuss the effect of the sample size and the number of iterations on the estimates of parameters $\mathcal{M}$.
A numerical case study as well as an engineering application involving determination of optimal spectral-bands for anomaly tracking is presented.
Optimal spectral-bands from EGO-MDA achieve at least 70\% improvement in the median deviance relative to other methods tested.

\end{abstract}

\noindent%
\vfill

\newpage
\spacingset{2} 

\section{Introduction}

Real-world processes are often distinguished based on the differences in their frequency spectra $S(f), f \in \mathbb{R}^+$. Example problems include classifying between machine vibrations under normal versus anomalous conditions, or between the voices of two people or animals. The differences between spectra, in turn, are quantified in terms of the differences in energy content
$E(\alpha_l) = \int_{f \in \alpha_l} S(f)df~\forall l \in [L]$
over certain spectral-bands $\underline{\alpha} = \{\alpha_l: \forall l \in [L]\}$, a set of disjoint intervals $\alpha_l \in f$. Identification of such spectral-bands $\underline{\alpha}$,
that are effective in discriminating between two classes of spectra, is termed the spectral-banding problem. 

\begin{definition}[Spectral-banding]
For two mean-square integrable stochastic processes $S^{(r)}(f)$, representing spectra belonging to two classes, $r = \{A,B\}$,
the spectral-banding problem is to determine the non-overlapping bands $\underline{\alpha} = \{\alpha_l: \forall l \in [L]\}$ from $N^{(r)}$ realizations of spectra from each class, such that the resulting energies $E(\alpha_l): \forall l \in [L], r \in \{A,B\}$ optimally discriminate (i.e. achieve minimum deviance  \parencite{hastie2009elements}) 
between the two classes of spectra. 
\end{definition}

As noted, the spectral-banding problem arises in many engineering and scientific contexts. 
In addition to enhancing classification accuracy, optimal spectral-bands also explain novel phenomena in underlying physical domains. For example, the identification of spectral-bands belonging to higher frequency can lead to the discovery of small-scale short-time phenomena as described in a case study of real-world application in Section 4.

Surprisingly, this problem has thus far eluded rigorous treatment. 
A few studies in mechanical systems literature have addressed the spectral-banding problem specifically for identifying time-localized impulses in signals for anomaly detection and diagnosis.
\textcite{WANG20111750} use spectral kurtosis, a statistical quantity measuring the signal's peakedness in the frequency domain. 
\textcite{moshrefzadeh2018autogram} identify spectral-bands that detect periodic and sharp energy peaks in the signal.
Such approaches do not aim to identify spectral-bands that discriminate between two different processes based on energy differences.
Additionally, an isolated work in speech processing literature by \textcite{6288298} outline a plausible approach for spectral-banding under certain band shape assumptions for a speech discrimination problem.
The spectral-bands are limited by the assumptions, such as covering the entire frequency range with triangular kernels and having the boundaries of the bands as the peak frequency of the neighboring bands.
These restrictions prevent the generalization of the approach to real-world industrial and engineering processes. Furthermore, their approach uses a gradient-based search strategy for identifying optimal spectral-bands. Gradients may not be available for complex systems and processes, limiting the extensibility of such approaches.

In Machine Learning literature, the spectral-banding problem reduces to a class of feature engineering problems for classification described by \textcite{guyon2003introduction}.
Pointwise hypothesis Testing approach developed by \textcite{cox2008pointwise} has been used to identify discriminating intervals.
\textcite{prakash2021gaussian} develop functional test statistics based on distribution assumption for realizations of different classes. These test statistics are used to identify the intervals that discriminate between the classes.
Hypothesis testing approaches are limited by assumptions on the shape and smoothness  of the spectra and cannot accommodate for non-stationarity that may exist in different intervals of the spectra. Also, the large number of enumerations makes the approach intractable and unsuitable for continuous spectra. 

\begin{figure}
    \begin{center}
    \includegraphics[width = 3 in]{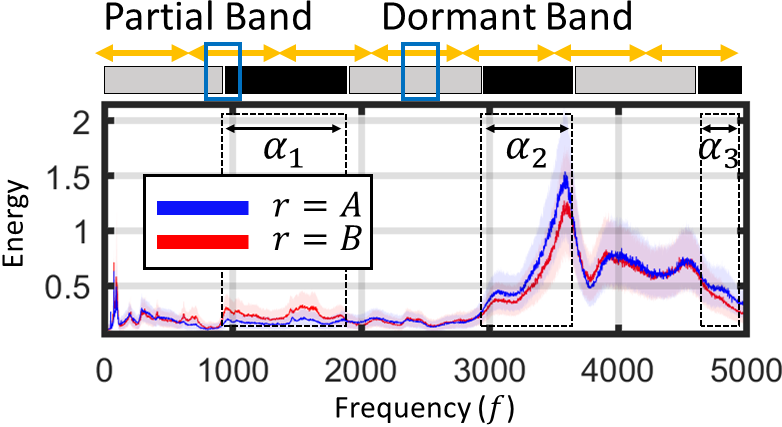}
    \end{center}
    \vspace{-0.1 in}
    \caption{Frequency spectra $S^{(r)}(f)$ for a manufacturing process under process condition $r=A$(blue) and $r=B$(red) with highlighted dormant and non-dormant bands. Uniform size bands are illustrated in yellow.
    }
    \label{fig:bird_mfg_spec}
\end{figure}

To further understand the underlying challenges, let us consider an example of three spectral-bands $\underline{\alpha} = \{\alpha_1, \alpha_2, \alpha_3\}$ that capture the differences between two classes of frequency spectra (see Figure~\ref{fig:bird_mfg_spec}) from a real-world process. These spectra occur at two different conditions, denoted by classes $r=A$ (blue) and $r=B$ (red), respectively. Silhouettes in Figure~\ref{fig:bird_mfg_spec} represent the energy variation. These spectral-bands $\underline{\alpha}$ (black bands) visually bear the most conspicuous differences between the spectra from the two classes, and the 
energy vector $\textbf{E}(\underline{\alpha}) = [ E^{(r)}(\alpha_1),$
$ E^{(r)}(\alpha_2), E^{(r)}(\alpha_3)]$ optimally classify between the two spectra. 
Intervals complement to $\underline{\alpha}$ are non-discriminating and can be considered dormant (grey bands). 
Sub-optimal banding, 
such as the use of fixed-size bands (illustrated in yellow), leads to the
inclusion
of spurious dormant bands (do not capture the relevant phenomenon) and partial bands (do not fully capture relevant bands). This can result in spectral-bands with poor signal-to-noise ratio (SNR) that lowers classification accuracy.

We present EGO-MDA, an unsupervised method to identify $\underline{\alpha}^*$ 
that overcomes limitations from assumptions on the shape of spectra and spectral-bands. EGO-MDA adopts a gradient-free sequential search strategy. This avoids an enumerative search which can otherwise become restrictive, especially in problems with very high sampling rates. 
The major contributions of our work are summarized as follows. 
(1) A rigorous solution approach of EGO-MDA is presented to address the spectral-banding problem. Although commonly encountered in several engineering domains, it has not received much attention in research in the signal processing and machine learning literature.
(2) The effects of certain hyperparameters on the performance of EGO-MDA are presented with benchmark spectral-banding problems for assessment. EGO-MDA outperforms other methods tested (with over 70\% lower deviance). 
Spectral-bands from EGO-MDA can further be used for real-time monitoring of numerous engineering processes, 
as well as for optimal channel allocation in diverse engineering and scientific domains. 

The remainder of this paper is organized as follows: 
In Section 2, we provide mathematical description for the spectral-bands and formulate the objective function for the spectral-banding problem. 
In Section 3, we present the EGO-MDA Algorithm and discuss theoretical considerations of sample size and iterations of the EM algorithm based on the separation between the Gaussian mixing components and mixture model misspecification. 
In Section 4, we discuss two case-studies of EGO-MDA for spectral-banding in a synthetic data set and a real-world application of the shell polishing process. 
Concluding remarks on the novelty of EGO-MDA in spectral-banding, performance limitations under pathological cases of true signals, improvements of the method through parameter tuning, 
and extension to image segmentation are discussed in Section 5.

\section{EGO-MDA for Spectral-Banding}

\noindent EGO-MDA for spectral-banding iteratively uses Mixture Discriminant Analysis (MDA) \parencite{hastie1996discriminant} to obtain the best classifier with model parameters $\mathcal{M}^*$ under fixed $\underline{\alpha}$,
and
an Efficient Global Optimization (EGO) \parencite{Jones1998} to sequentially search for the set of optimally discriminating bands $\underline{\alpha}^*$.
EGO-MDA performance is sensitive to several hyperparameters which can affect the quality of the optimal spectral-bands $\underline{\alpha}^*$ obtained. This section introduces statistical modeling of spectral-band energy $\textbf{E}(\underline{\alpha})$, criteria for assessing spectral-band quality and pathological cases often encountered in practice which affect EGO-MDA performance.
An example of the spectral-bands $\underline{\alpha}=\{\alpha_1,\ldots,\alpha_l,\ldots,\alpha_L\}$ and the region of the frequency spectra $S(f)$ corresponding to spectral-band energy $\{E(\alpha_1),\ldots,E(\alpha_l),\ldots,E(\alpha_L)\}$ is illustrated in Figure~\ref{fig:band_overlay}.

\begin{figure}
\begin{center}
\includegraphics[width = 6 in,height = 2 in]{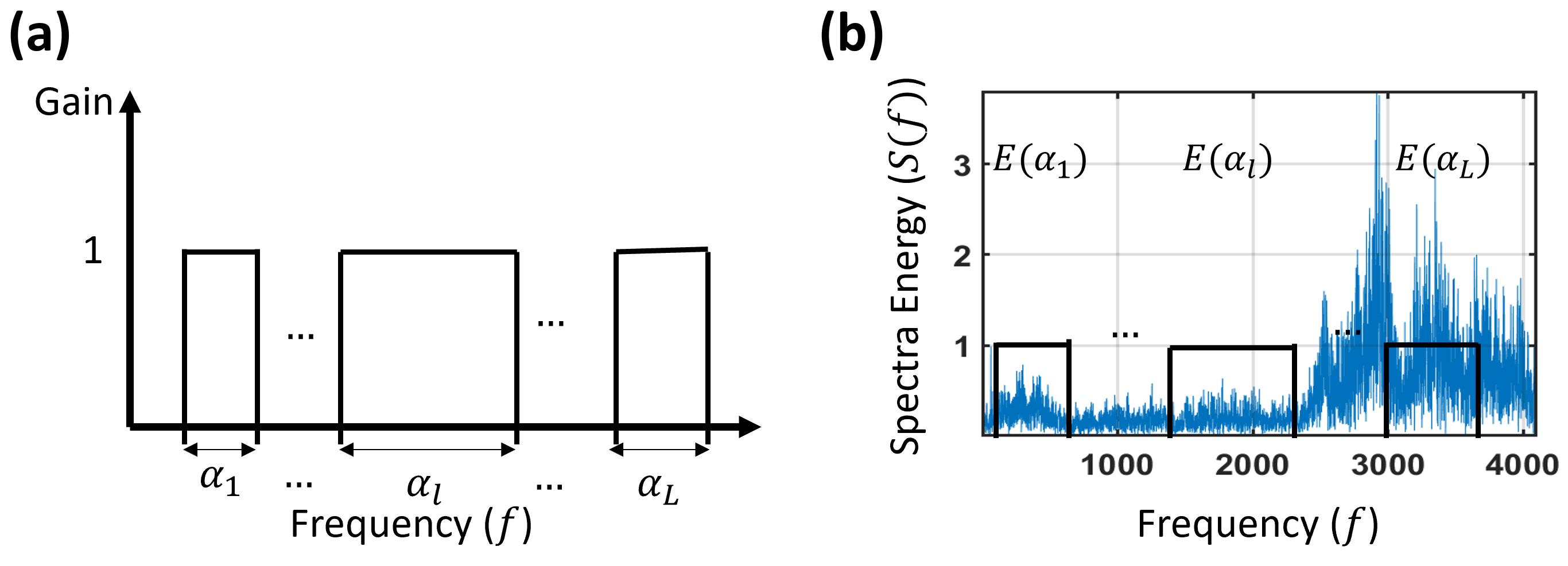}
\end{center}
\vspace{-0.1 in}
\caption{Schematic of spectral-bands and spectral-band energy (a) Spectral-bands $\underline{\alpha}$ in the frequency domain.
(b) Overlay of spectral-bands $\underline{\alpha}$ on frequency spectra $S(f)$.}
\label{fig:band_overlay}
\end{figure}

Here, energy $\textbf{E}^{(r)}(\underline{\alpha})$ in spectral-bands of class $r$ is modeled using multivariate $d$-dimensional GMM distributions

\vspace{-0.3 in}
\begin{gather}
    \textbf{E}^{(r)}(\underline{\alpha}) \mid r 
    \sim
    \sum_{k=1}^{K}\pi^{(r)}_{k}\mathcal{N}({\boldsymbol\mu}^{(r)}_k,\Sigma)    
    \label{eq: GMM_dist}
\end{gather}

\noindent with $K$ mixing components, and the model parameters given by

\vspace{-0.3 in}
\begin{gather}
    \mathcal{M}^{(r)} 
    =
    \{\pi^{(r)}_k, {\boldsymbol\mu}^{(r)}_k, \Sigma | k \in \{1,\ldots,K\}\}
    \label{eq:GMM_param_set}
\end{gather}

\noindent where $\pi^{(r)}_k$ and ${\boldsymbol\mu}^{(r)}_k$ are the weight and $d$-dimensional mean vector for the $k^{\rm th}$ mixing component, respectively for class $r$. $\Sigma$ is the common covariance matrix for each mixing component. Centroids of the $k$ mixing components in class $r$ are collectively denoted by $\boldsymbol\mu^{(r)} = \{\boldsymbol\mu^{(r)}_1,\ldots,\boldsymbol\mu^{(r)}_K\}$.

EGO-MDA seeks spectral-bands $\underline{\alpha}$ that minimizes adjusted \textit{deviance} given by  

\vspace{-0.3 in}
{\begin{gather}
    \mathcal{D}(\underline{\alpha},\mathcal{M})
    =
    -2\sum_{r = \{A,B\}}\sum{\rm log} \hat{p}(r|\textbf{E}^{(r)}(\underline{\alpha})) + \eta{\rm log}(|\underline{\alpha}|)
    \label{deviance}
 \end{gather}

}

\noindent as an indicator of the discrimination between two classes. Here, $\hat{p}(r|\textbf{E}^{(r)}(\underline{\alpha}))$ is the class probability for the realization $\textbf{E}^{(r)}(\underline{\alpha})$, $|\underline{\alpha}|$ is the total sum of all spectral-band widths, $\eta$ is the penalty coefficient, and $\mathcal{M} = \mathcal{M}^{(A)} \cup \mathcal{M}^{(B)}$. The penalty coefficient $\eta$ is introduced to favor compact spectral-bands for achieving a desirable spectral-band configuration, i.e., larger values of $\eta$ favor bands of smaller size.

EGO-MDA essentially addresses the spectral-banding problem by iterating between problems of solving $\underset{\underline{\alpha}}{\rm argmin} ~\mathcal{D}(\underline{\alpha},\mathcal{M})$ using EGO search and $\underset{\mathcal{M}}{\rm argmin}~\mathcal{D}(\underline{\alpha},\mathcal{M})$ by fitting GMM parameters $\mathcal{M}$ using MDA.
The response surface of $\mathcal{D}(\underline{\alpha},\mathcal{M})$ can have multiple local optima as observed by \textcite{jin2016local} for log-likelihood surfaces. EGO-MDA can be effective in such scenarios because of the exploration/exploitation trade-off in EGO.

MDA uses the EM algorithm to obtain MLE of the GMM model parameters $\mathcal{M}$. 
Hyperparameters affecting the estimation of $\mathcal{M}$ are initial centroid $\boldsymbol\mu^{(r)0}_k$, assumptions on $\Sigma$ and $K$, and  $\pi^{(r)}_k$. User-defined parameters including  ${\rm N}^{(r)}$ and the number of iterations $t$ of the EM algorithm can also directly affect the estimation of $\mathcal{M}$. Pre-mature termination of the iterations 
provide inaccurate estimates and convergence guarantee cannot be provided for insufficient sample size.

We study two scenarios commonly encountered in practice that affect the performance of the EGO-MDA algorithm: (i) separation of the centroids, and (ii) initial specification of the number of centroids.

\vspace{-0.25 in}
\subsection{Separation of Mixtures}

\noindent Reducing separation in terms of distance between centroids $\rho^{(r)}_{ij}$ of mixing components $k=i$ and $k=j$ in class $r$ demand an increasing number of samples ${\rm N}^{(r)}$ and iterations ${\rm t}$ of the EM algorithm for convergence to accurate estimates.
The centroid parameter estimate of iteration ${\rm t}$ is indicated by superscript as $\boldsymbol\mu^{(r)t}$.
The estimate convergence in terms of error ${\rm E}(\boldsymbol\mu^{(r)})=\underset{k}{\rm max}||\boldsymbol\mu_k^{(r)}-\boldsymbol\mu_k^{(r)*} ||$ is captured by the following proposition.

\begin{figure*}
\begin{center}
\includegraphics[width=\textwidth]{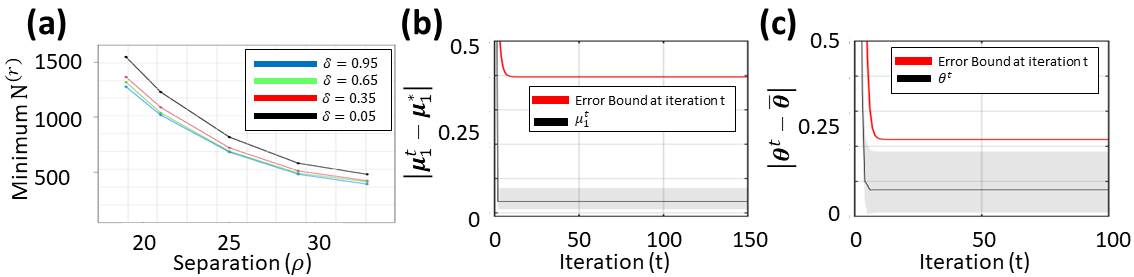}
\end{center}
\vspace{-0.1 in}
\caption{(a) Variation of the minimum sample size requirement with $\rho$ and $\delta$ per Proposition 1.
(b) Variation in estimation error of ${\boldsymbol\mu}_1$ for small $\rho$. (c) Variation in estimation error of ${\boldsymbol\theta}$ when $K$ is misspecified.
}
\label{fig:num_sim_1}
\end{figure*}

\begin{theorem}

Let $R^{(r)}_{max} = \underset{i \neq j}{\rm max}\rho^{(r)}_{ij}$,$R^{(r)}_{min} = \underset{i \neq j}{\rm min}\rho^{(r)}_{ij}$, $\pi^{(r)}_{min}=\underset{k}{\rm min}\pi^{(r)}_k$, $R_k^{(r)} = \underset{j \neq k}{\rm min}\rho^{(r)}_{jk}$ 
and $\mathcal{U}^{(r)}_{\lambda} = \{\boldsymbol\mu^{(r)}:||\boldsymbol\mu_k^{(r)}-\boldsymbol\mu_k^{(r)*}|| \leq \lambda R_k^{(r)}~\forall k \}$ and if

\vspace{-0.25 in}
\begin{gather}
\frac{{\rm N^{(r)}}}{{\rm log}({\rm N^{(r)}})} 
>
C\frac{Kd{\rm log}(\frac{\Tilde{C}}{\delta})}{\pi^{(r)}_{\rm min}}{\rm max} \bigg( 1,\frac{1/\lambda^2\pi^{(r)}_{\rm min}(R^{{(r)}}_{\rm min})^2}{(1-2\lambda)^2} \bigg)
\label{eq:ss_req}
\end{gather}

\noindent for some suitable constant $C$,$C_1$, $\Tilde{C} = 100K^2R^{(r)}_{max}(\sqrt{d}+2R^{(r)}_{max})^2$, $\lambda \in (0,\frac{1}{2})$ and $\delta \in (0,1)$, then with probability at least 1-$\delta$, for all iterations t, ${\boldsymbol\mu}^{(r) \rm t} \in \mathcal{U}^{(r)}_{\lambda}$ 

\vspace{-0.3 in}
{\begin{gather}
        {\rm E}(\boldsymbol\mu^{(r)}_k) \leq
         \frac{1}{2^{\rm t}}\underset{k}{\rm max}{\rm E}({\boldsymbol\mu}_k^{(r) 0}) + \frac{C_1/\pi^{(r)}_k}{(1-2\lambda)}\sqrt{\frac{{\rm log}(\frac{\Tilde{C}{\rm N^{(r)}}}{\delta})}{{\rm N^{(r)}}/Kd}}
        \label{eq:convg_1}
    \end{gather}
}

\end{theorem}

\begin{proof}
In MDA, the GMM parameters are estimated independently for each class $r$. Therefore Equation (\ref{eq:convg_1})
follows from Theorem 3.3 in \textcite{segol2021improved} presented for the case of GMM.
\end{proof}

The proposition provides convergence guarantees and guidance on the choice of ${\rm N}^{(r)}$. The convergence bounds on estimation error ${\rm E}(\boldsymbol\mu^{(r)}_k)$ would guide the adjustment of sample size ${\rm N}^{(r)}$ and the iterations of the EM algorithm ${\rm t}$. 

The following numerical example illustrates Proposition 1. $X_1,\ldots,X_{{\rm N}^{(r)}} \in \mathbb{R}$ ($d = 1$) are drawn from a GMM with $K=2$, $\pi_1 = \pi_2 = 0.5$, unit variance and true centroids ${\boldsymbol\mu}^*_1=3$, ${\boldsymbol\mu}^*_2=21$. The convergence result in Equation (\ref{eq:convg_1}) holds when the minimum requirement on ${\rm N^{(r)}}$ specified by Equation (\ref{eq:ss_req}) is satisfied.  Figure~\ref{fig:num_sim_1}(a) illustrates the requirement on ${\rm N}^{(r)}$ with increasing separation $\rho$. Following this, we choose sample a size of ${\rm N}^{(r)}=1500$. We set $\lambda = 0.05$ and $\delta=0.05$, which ensures conditions for convergence in proposition 1 are satisfied. The number of iterations required for this convergence is guided by the bound in “red” in Figure~\ref{fig:num_sim_1}(b). The example illustrates estimate convergence for a sufficiently large number of iterations. 

\vspace{-0.25 in}
\subsection{Mixture Model Misspecification}

\noindent In the scenario of misspecified GMM, specifically the case of underspecified GMM, fewer mixing components $K$ than of the true distribution are considered. Underspecified models are often encountered when the mixing components are closely located.

We consider the true distribution $g^{*}$ and the model $g$ being fit given by 

\vspace{-0.3 in}
\begin{flalign}
\begin{split}
    g^*&=\sum_{\omega \in \{1,2\}}\frac{1}{4}\mathcal{N}(\theta^*((-1)^\omega\rho-1),1) + \frac{1}{2}\mathcal{N}(\theta^*,1)
    \label{sim2true}
\end{split}
\end{flalign}

\vspace{-0.3 in}
\begin{flalign}
\begin{split}
    g & =
    \frac{1}{2}\mathcal{N}(\theta,1) + \frac{1}{2}\mathcal{N}(-\theta,1)
    \label{model2}
\end{split}
\end{flalign}

\noindent where in Equation (\ref{sim2true}), the true distribution $g^*$ with $K=3$ has true centroid as a function of $\theta^*$ and separation $\rho$. The model $g$ being fit in Equation (\ref{model2}) however has $K=2$. Both distributions $g$ and $g^*$ have unit variance ($\sigma = 1$) for each of the Gaussian components. 

In the underspecified case, \textcite{dwivedi2018theoretical} show that estimates of the EM algorithm instead converge to ${\bar{\theta}} \in \underset{\theta}{\rm argmin}~{\rm KL}(g^*,g)$. The following proposition captures this. 

\begin{theorem}
Let ${\rm N}^{(r)} \geq  c_1{\rm log}(\frac{1}{\delta})$ for $\delta \in (0,1)$;  $\xi = \frac{|\theta^*|}{\sigma}$, $\mathbb{B}(\theta,R)$ as a ball of radius $R$ centered at $\theta$, $M_n:\mathbb{R} \mapsto \mathbb{R}$ be the EM operator, and  $\theta^0 \in \mathbb{B}(\bar{\theta},\frac{|\bar{\theta}|}{4})$. Under the assumptions of Theorem 1 in \textcite{dwivedi2018theoretical}, with probability at least 1-$\delta$, the sample-based EM sequence $\theta^{t+1} = M_n(\theta^t)$ for model (\ref{sim2true}) satisfies 

\vspace{-0.3 in}
\begin{gather}
    |\theta^t - \bar{\theta}|
    \leq
    \gamma^t|\theta^0-\bar{\theta}|+\frac{c_2}{1-\gamma}|\theta^*|((\theta^*)^2+\sigma^2)\sqrt{\frac{{\rm log}{\frac{1}{\delta}}}{{\rm N}^{(r)}}}
\end{gather}
\label{eq:convg_2}

\vspace{-0.3 in}
\noindent
where $\gamma \leq e^{-c'\xi^2}$, $c_1$, $c_2~{ \rm and}~c'$ are suitable constants.
\end{theorem}

\begin{proof}
In MDA, GMM parameters for each class $r$ is estimated independently. Therefore, Corollary 1 of \textcite{dwivedi2018theoretical} presented for GMM extends to the case of MDA. 
\end{proof}

Two major implications follow in the case of misspecified models. First, due to incorrect model specification, the estimates $\theta^t$ converge to $\bar{\theta}$ and not the true centroid $\theta^*$. 
However, $\bar{\mu}$ still lies in the neighborhood of the true parameter value and this neighborhood shrinks with reducing $\rho$. Second, the convergence bound gets tighter with increasing sample size and varies in the order of $n^{- \frac{1}{2}}$.
The following numerical example illustrates the convergence result in Proposition 2. We choose $\theta^*=1$, $\rho =0.04$ thus $\eta = 1$,$C_{\rho}=0.49$ and $\bar{\theta} \in [0.51,1.49]$. ${\rm N}^{(r)}=1500$ samples are generated from true distribution $g^*$. For $\theta^0 = 0.6$ and $\delta = 0.05$, the convergence of the estimates of the centroid is illustrated for several EM algorithm runs in Figure~\ref{fig:num_sim_1}(c). The bound at each iteration corresponding to Equation \ref{eq:convg_2} is given by the plot in ``red".

The following key insights result from the two numerical examples for the estimation of model parameters $\mathcal{M}$ for MDA using the EM algorithm.
(i) The choice of sample size to guarantee convergence is determined by the separation of the Gaussian centroids in the true data-generating distribution. (ii) The two numerical examples provide us guidance on the choice of the number of iterations to run the EM algorithm. (iii) In the case of misspecified models, the estimates do not converge to the true centroids. (iv) In the misspecified case, new centroid $\bar{\theta}$ lies in the neighborhood of the true centroid $\theta^*$, and neighborhood size is determined by the separation between the centroids of the Gaussian mixtures.

\section{EGO-MDA Algorithm}

The solution approach in EGO-MDA iterates between two steps. In the first step, spectral-bands $\underline{\alpha}$ are identified using EGO. In the second step, GMM parameters $\mathcal{M}$ in the MDA problem are estimated for said $\underline{\alpha}$ using the EM algorithm. Steps in EGO-MDA are summarized in  Algorithm \ref{alg:EGO-MDAb}. Here,
$\rm{N}^{(r)}$ realizations of spectra $S^{(r)}(f)$ are used as input for class $r$.

\begin{algorithm}

    \SetKwData{Left}{left}\SetKwData{This}{this}\SetKwData{Up}{up}
    \SetKwFunction{MDA}{MDA}\SetKwFunction{EGO}{EGO}
    \SetKwInOut{Input}{Input}\SetKwInOut{Output}{Output} 
    
    \Input{${\rm N}^{(r)}$ realizations of $S^{(r)}(f), r \in \{A,B\}$}
    Initialize $Q = 0$, $P = N^{init}$, Budget constraint B\;
    Initialize $\zeta^Q$ from $(\underline{\alpha}^p,y^p)~\forall p \in [P]$\;
    
    \While{{\rm B} is satisfied}{
        $P = P+1$, $Q = Q+1$\; 
        Set $\underline{\alpha}^{P} = \underset{\underline{\alpha}}{\rm{argmax}}~\rm{EI}(\underline{\alpha})$\;
        Calculate $\textbf{E}(\underline{\alpha}^P)$\;
        Calculate $y^{P} = \mathcal{D}(\underline{\alpha}^P,\mathcal{M}^{P*})$ by fitting $\mathcal{M}^{P*}$\;
        Estimate GPR model $\zeta^Q$\;
    }
    set $(\underline{\alpha}^*,y^*) = (\underline{\alpha}^{P_*},y^{P_*}) \mid P_* = \underset{p}{\rm argmin}~y^p$\;
    
    \Output{$\underline{\alpha}^*$}
    
\caption{EGO-MDA}\label{alg:EGO-MDAb} 
\end{algorithm}

In each iteration $p$ of EGO-MDA, we consider a tuple $(\underline{\alpha}^p,y^p)$, where $\underline{\alpha}^p$ represents the spectral-band and $y^p = \mathcal{D}(\underline{\alpha}^p,\mathcal{M}^{p*})$ represents the corresponding deviance.
EGO fits a Gaussian Process Regression (GPR) model $\zeta$ to the history of observed tuples $(\underline{\alpha}^p,y^p)$.
The GPR model $\zeta$ 
comprises a multivariate Gaussian prior, updated using a Bayesian mechanism, resulting again in a multivariate Gaussian posterior.
An initial GPR model $\zeta^0$ is developed from $(\underline{\alpha}^p,y^p)~\forall p \in [{\rm N}^{init}]$.
Upper case notations --- $P$ and $Q$ --- represent total spectral-bands (initial spectral-bands and spectral-bands identified from EGO search) and iteration of the EGO search respectively that are updated in each iteration of the EGO search. Lower case notations --- $p$ and $q$ --- are variables used as identifiers for spectral-bands and GPR models, respectively.

In EGO-MDA iterations, at step $\textbf{5}$, estimated GPR model $\zeta^p$ is used to determine the next point $\underline{\alpha}^{p+1}$, such that the expected improvement,  ${\rm EI}(\underline{\alpha}^{p+1}) = {\rm E}[{\rm max}({\rm min}(y^p)-\mathcal{D}(\underline{\alpha}^{p+1},\mathcal{M}^{(p+1)*}),0)]$ is maximized. \textcite{Jones1998} show that the closed-form representation of ${\rm EI}(\underline{\alpha}^{p+1})$ is a function of the GPR posterior distribution parameters. 

The minimization ${\mathcal{M}^*} = \underset{\mathcal{M}}{\rm argmin}~\mathcal{D}(\underline{\alpha},{\mathcal{M}})$ in step $\textbf{7}$ of Algorithm \ref{alg:EGO-MDAb} is solved by fitting a Mixture Discriminant Analysis (MDA) classifier. 
Initialization of ${\boldsymbol\mu}^{(r)}_k, \pi^{(r)}_k$ is obtained from K-means clustering as described in \textcite{hastie1996discriminant}. MDA classifier is obtained as the MLE of the model parameters $\mathcal{M}^{(r)}=\{\pi^{(r)}_k, {\boldsymbol\mu}^{(r)}_k, \Sigma | r = \{A,B\}, k \in [K] \}$. Accuracy of estimates $\mathcal{M}$ in step $\textbf{7}$ deteriorates in case of a misspecified mixture model or low separation between centroids as discussed in Section 2. 

The ($\underline{\alpha}^p,y^p$) obtained from steps \textbf{5-8} is used to update the GPR model $\zeta^q$. 
This completes one iteration of the EGO search in steps $\textbf{4-8}$. 
The iterations stop when budget constraint ${\rm B}$ specified based on the number of iterations or improvement over the deviance is satisfied.
\vspace{-0.35 in}
\section{Case Studies}

\noindent We consider two case studies to assess the performance of EGO-MDA. The first study uses a synthetic data set that meets the assumptions underpinning EGO-MDA to assess the performance relative to the ground truth.
The second case study compares EGO-MDA relative to other contemporary methods in an industrial application.

\vspace{-0.25 in}
\subsection{Synthetic Data Set}

\noindent Let us consider ${\rm{N}^{(A)}} = {\rm{N}^{(B)}} = 1000$ realizations each of two spectra defined by the GMMs 

\vspace{-0.3 in}
\begin{align}
    S^{(A)}(f) \sim \delta_{f \in \alpha^*_1}[\frac{4}{5}\mathcal{N}(0,1) + \frac{1}{5}\mathcal{N}(69 - 16|f-23|,1)]
    + \delta_{f \notin \alpha^*_1} [\mathcal{N}(0,1)]
\label{eq:classA_sim}
\end{align}
\vspace{-0.5 in}
\begin{align}
    S^{(B)}(f) \sim \delta_{f \in \alpha^*_2}[\frac{4}{5}\mathcal{N}(0,1) + \frac{1}{5}\mathcal{N}(84-16|f-53|,1)]
    +
    \delta_{f \notin \alpha^*_2}[\mathcal{N}(0,1)]
\label{eq:classB_sim}
\end{align}

\begin{figure}
\begin{center}
\includegraphics[width = 3 in,height = 2.5 in]{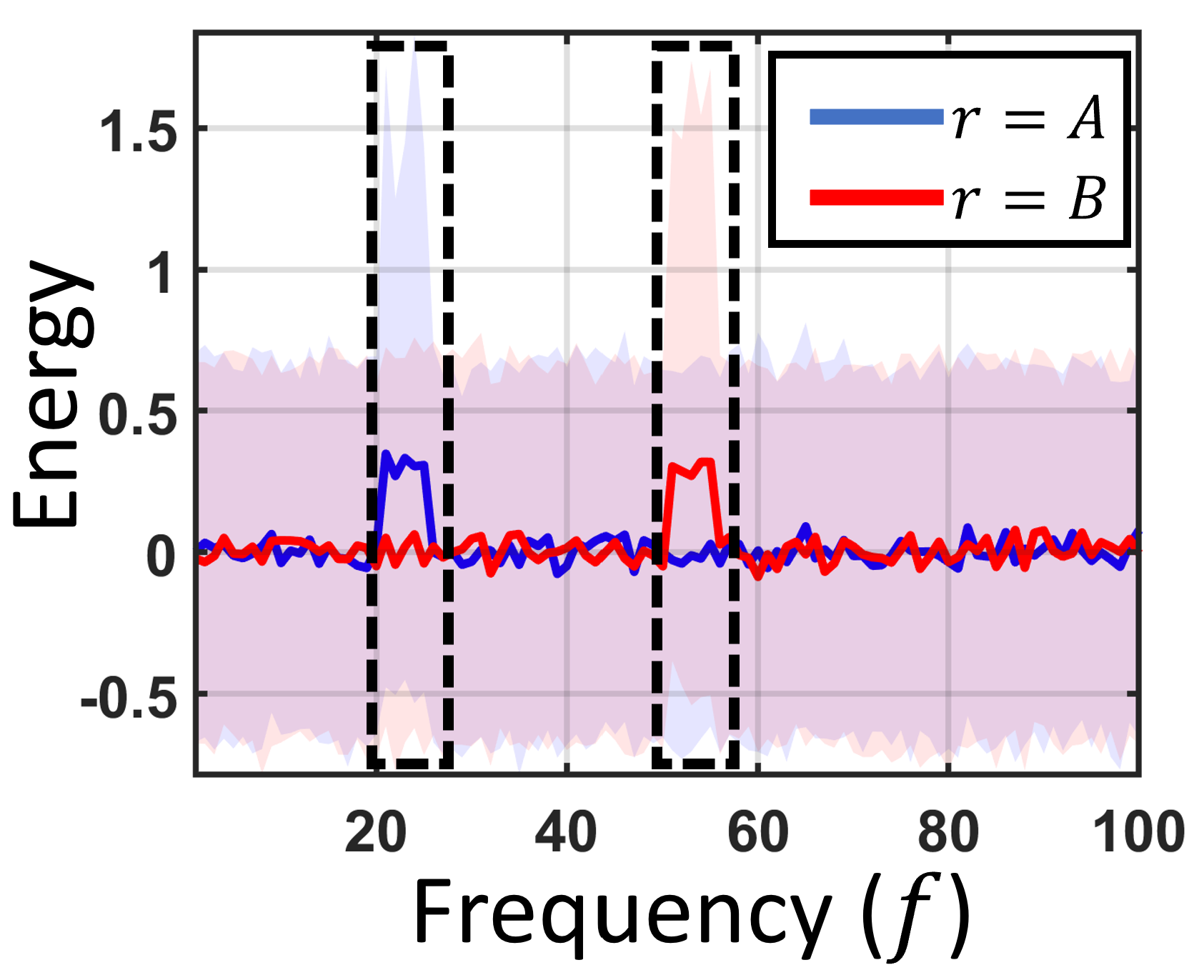}
\end{center}
\vspace{-0.1 in}
\caption{ 
Simulated frequency spectrum for two classes $r=\{A,B\}$ with median energy illustrated as line plot and silhouettes representing the $25^{th}$ and $75^{th}$ percentile energy. Black broken boxes indicate the true location of $\alpha^*_1$ and $\alpha^*_2$.
}
\label{fig:simulated_spectra}
\end{figure}

\begin{figure*}
\begin{center}
\includegraphics[width = \textwidth]{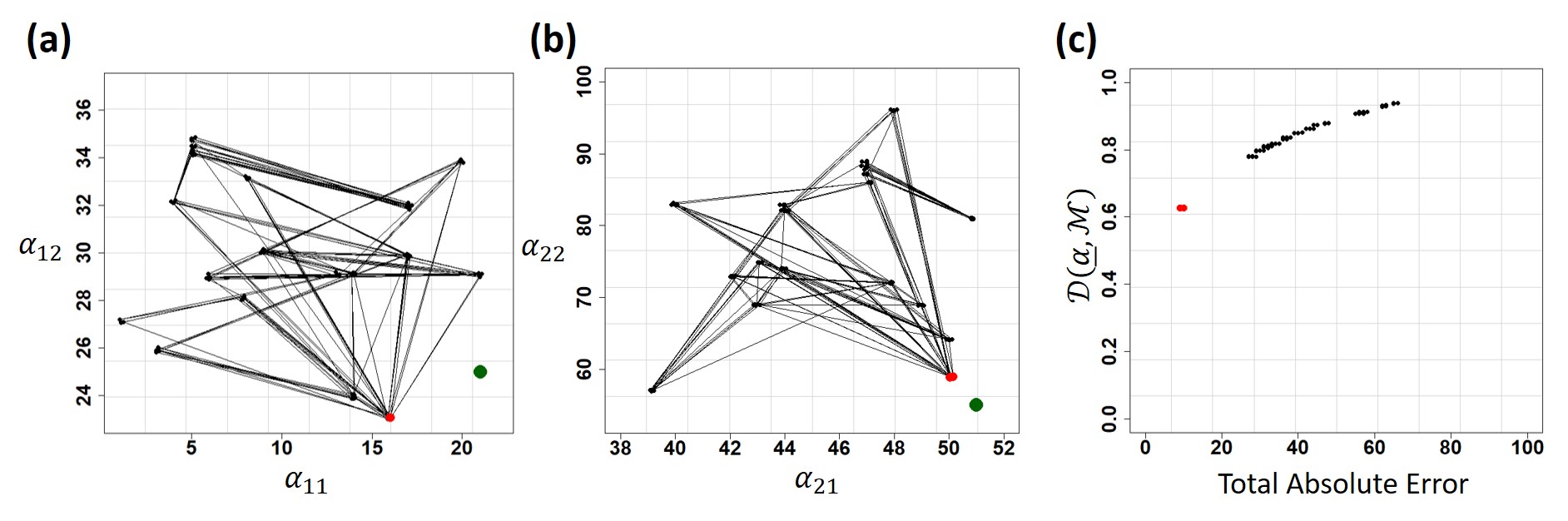}
\end{center}
\vspace{-0.1 in}
\caption{ 
EGO-MDA solutions with data points in red indicating the minimum deviance solutions (a) Optimal solution search for spectral-band $\alpha_1$. The true solution ($\alpha^*_{11},\alpha^*_{12}$) is represented as a green data point (b) Optimal solution search for spectral-band $\alpha_2$. The true solution ($\alpha^*_{11},\alpha^*_{12}$) are represented as a green data point (c) Relation between deviance and total absolute error $\sum_{i = \{1,2\}}|\alpha_{i1}-\alpha^*_{i1}|+|\alpha_{i2}-\alpha^*_{i2}|$ indicating deviance to be an appropriate choice as an objective. Spectral-bands achieving minimum deviance also achieve minimum error and are therefore relatively accurate with respect to spectral-band width and location compared to spectral-bands with high deviance.
}
\label{fig:EGO_simulation_solution}
\end{figure*}

\noindent where $\alpha^*_1 = [21,25]$ Hz,  $\alpha^*_2=[51,55]$ Hz, and $\delta_{(\cdot)}$ is the Kronecker delta. Spectral-bands $\alpha^*_1$ and $\alpha^*_2$ indicate the changes taking place in the spectrum when comparing between the classes $r$. The two boundaries of spectral-band $\alpha^*_i$ are represented by $\alpha^*_{i1}$ and $\alpha^*_{i2}$ respectively, such that $\alpha^*_{i1} < \alpha^*_{i2}$. As an example for $\alpha^*_1$, $\alpha^*_{11} = 21$ and $\alpha^*_{12} = 25$. For $\alpha^*_2$, $\alpha^*_{21} = 51$ and $\alpha^*_{22} = 55$. Figure~\ref{fig:simulated_spectra} illustrates the simulated spectra.

EGO-MDA aims to determine the spectral-bands $\underline{\alpha} = \{\alpha_1,\alpha_2\}$ that minimize the deviance in Equation (\ref{deviance}). We set $\eta=1/{\rm log}w_0$ , where $w_0$ is the maximum deviation possible (here, $w_0=100$ was set). This choice of $\eta$ is chosen to appropriately scale the penalty term between 0 and 1.
The hyperparameters of the GPR model $\zeta$ in EGO-MDA were chosen to balance exploitation and exploration as described in \textcite{ginsbourger2015package}. Accordingly, the Gaussian covariance kernel was employed with range parameters $\theta_R = 0.99$ and $\sigma_R = 9$.
In this implementation, EGO-MDA was run for 200 iterations, yielding 200 spectral-banding configurations $\{\underline{\alpha}^p:p \in \{1,\ldots,200\}\}$.

EGO-MDA solutions for each spectral-band $\alpha_i, i = \{1,2\}$ are presented as ($\alpha_{11}$,$\alpha_{12}$) and ($\alpha_{21}$,$\alpha_{22}$) representing the spectral boundaries in Figure~\ref{fig:EGO_simulation_solution}(a) and Figure~\ref{fig:EGO_simulation_solution}(b), respectively. Correspondingly, true spectral-bands $\alpha^*_1$ and $\alpha^*_2$ are illustrated as green data points. 
The solutions closer to the green dots are more accurate (with minimum error) spectral-band solutions than the farther solutions.
The total absolute error 
in spectral-bands is defined as the sum of the absolute difference between the boundaries of $\alpha_i$ and the true bands $\alpha^*_i$.
The total absolute error is given as $\sum_{i = \{1,2\}}|\alpha_{i1}-\alpha^*_{i1}|+|\alpha_{i2}-\alpha^*_{i2}|$. 
The top 5 percentile spectral-bands based on minimum values of the deviance $\mathcal{D}(\underline{\alpha},\mathcal{M})$ are represented as red data points in Figure~\ref{fig:EGO_simulation_solution}. These top 5 percentile spectral-bands also achieve minimum total absolute error compared against true spectral-bands $\alpha^*_1$ and $\alpha^*_2$ as illustrated in Figure~\ref{fig:EGO_simulation_solution}(c). Therefore indicating minimization of deviance achieves ground truth spectral-bands. We note that the minimum deviance $\mathcal{D}(\underline{\alpha},\mathcal{M})$ is $0.5$ (see Figure~\ref{fig:EGO_simulation_solution}(c)) because $w_0=100$ in Equation \ref{deviance}.

Search by exploitation in EGO-MDA is evident from the clustering of solutions observed in Figure~\ref{fig:EGO_simulation_solution}(a) and Figure~\ref{fig:EGO_simulation_solution}(b).
EGO-MDA identifies several solution neighborhoods (clusters) and exploits each neighborhood to identify the local optima. Identification of several neighborhoods indicates the exploration behavior of EGO-MDA.
A histogram plot of the L2-norm distance between subsequent solutions identified during EGO-MDA search is presented in Figure~\ref{fig:distance_exploit}.
The largest frequency in the histogram lies closer to 0 thus indicating over 35\% of the subsequent solutions are adjacent, further indicating exploitation within a neighborhood.
Thus, the EGO-MDA search balances between exploration and exploitation.

\begin{figure}
\begin{center}
\includegraphics[width = 6 in,height = 1.5 in]{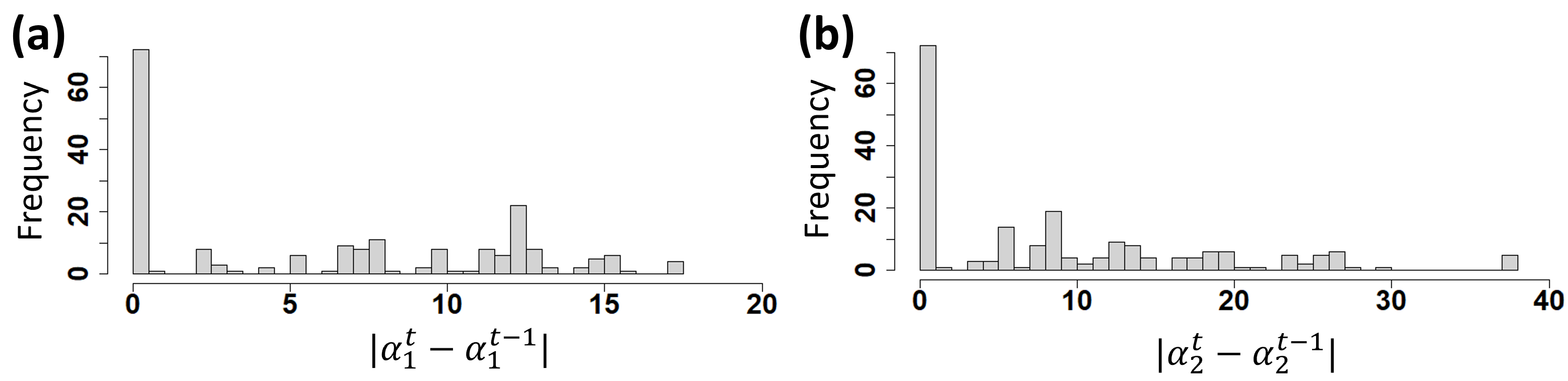}
\end{center}
\vspace{-0.1 in}
\caption{ 
Histogram of distance (L2 norm) between solutions identified for subsequent iterates of the EGO-MDA algorithm (a) $\alpha_1$ (b) $\alpha_2$.
}
\label{fig:distance_exploit}
\end{figure}

Out of the 200 spectral-bands identified, all spectral-bands overlap with the dormant bands. 52\% of spectral-bands have complete overlap with true bands $\underline{\alpha}^*$ and remaining 48\% have partial overlap with at least one true band $\alpha^*_1$ or $\alpha^*_2$. 
As introduced in Section 1, the correct bandwidth for spectral-bands is important in maintaining high SNR. While complete overlap with true bands entirely captures the discriminating regions, spectral-bands that are exceeding large compared to the size of true bands $\underline{\alpha}^*$ can have poor SNR. It is noted that the average bandwidth of 52\% of spectral-bands having complete overlap with true bands exceeds the true bandwidth by a factor of 4.36. 
In comparison, the average bandwidth of the top 5 percentile bands with minimum deviance exceeds the true bandwidth by a factor of 0.78.
Such a large deviation in bandwidth is observed primarily because of the setup of the simulation with only 100 frequency points and the true spectral-bands $\underline{\alpha}^*$ occupying an even smaller proportion (10\%) of the frequency domain.
However, in practice, spectral-bands span over a larger frequency range as seen in the following case study, so the factor of error is relatively small.

\vspace{-0.25 in}
\subsection{Spectral-Banding for Shell Polishing Process}

\noindent A case study of the EGO-MDA algorithm for anomaly detection in a polishing process to finish 2mm diameter spherical ceramic shells with high-density carbon coating  is presented. 
These polished spherical shells (see Figure \ref{fig:shell_defect} (a)) are a critical component of the novel Inertial Confinement Fusion process being investigated at the Lawrence Livermore National Laboratory (LLNL).
A Record fusion yield of 72\% was achieved in 2021 at the National Ignition Facility. \textcite{zylstra2022experimental} note that this record yield was achieved due to the significantly higher surface quality of the shells --- small number of pits and voids.
Anomalies such as pitting and cracking of shells during polishing are known to be a major quality challenge during the polishing process (see Figure \ref{fig:shell_defect} (b,c)). 

\begin{figure*}
\begin{center}
\includegraphics[width = \textwidth]{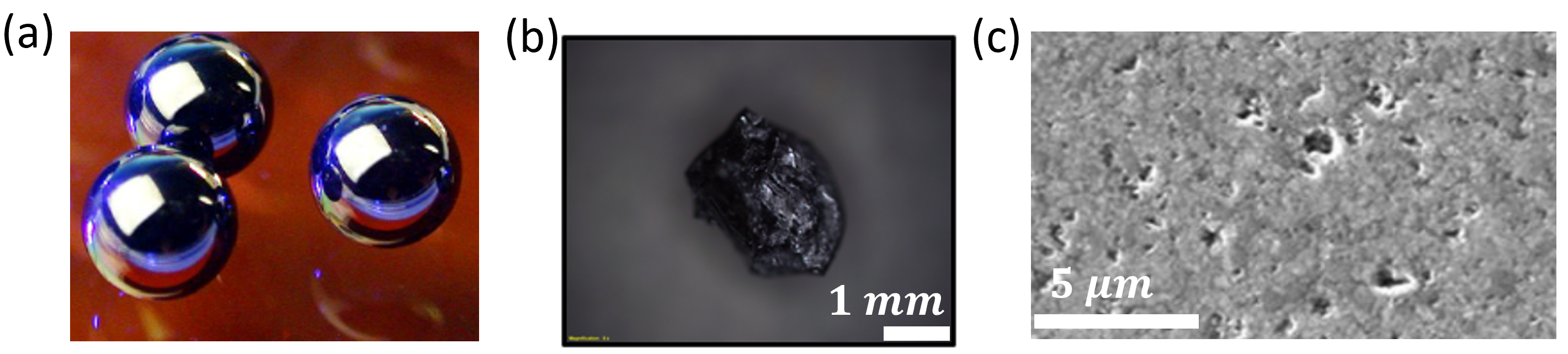}
\end{center}
\vspace{-0.1 in}
\caption{ 
(a) Finished Shells (b) Optical micrograph of shell debris, which can lead to damage of shells during polishing (c) Scanning electron micrograph of pits on the polished shell surface. (Figures (a) and (c) are reproduced from \textcite{HuangDefects2016})
}
\label{fig:shell_defect}
\end{figure*}

In-situ polishing process monitoring can help with the prognosis of events leading to anomalies in the polishing process which can further result in deterioration of surface quality. 
The spectral-bands which can discriminate spectra between the polishing of normal and damaged shells can be used for monitoring and detection of incipient anomalies. 
LLNL attached vibration sensors to their polishing machine to provide a cost-effective and reliable monitoring of the process for surface finish $S_a$ \parencite{shilan2022} and especially to detect events such as collisions.
Vibration signals are known to allow early detection of anomalies. This is because anomalies such as cracking are known to modify the spectral characteristics of the vibration signals. It is therefore important to identify the spectral-bands that capture the onset of an anomaly. In this case study, we compare the performance of EGO-MDA relative to three other methods in determining these spectral-bands.

Here, we collected  ${\rm N}^{(r)} = 600$ spectra samples of $S^{(r)}(f)$ corresponding to vibration signals from a polishing process in two process operation conditions --- before  ($r=A$) and after ($r=B$) the onset of an anomaly. Indeed, the assumption is that the process underlying each sample spectrum is stationary so that the resulting spectral-band energies can be modeled as GMM.

The performance of EGO-MDA for spectral-banding is compared against three other methods.
(1) \textit{RF-MDA}: Spectral-bands of uniform size are defined, and a Random Forest (RF) classifier is developed to determine the important spectral-bands based on high values of the GINI index \parencite{hastie2009elements}. Deviance of these uniform spectral-bands is identified by developing an MDA classifier. 
(2) \textit{R-MDA}: Spectral-bands of non-uniform size are chosen based on random initialization of bandwidth and location within a predetermined neighborhood. These randomly initialized bands are also used as initialization for EGO-MDA. 
(3) \textit{NM-MDA}: Important spectral-bands from RF-MDA are optimized using Nelder-Mead (NM) simplex optimization. Comparison against elementary approach such as uniform spectral-banding are ignored due to their inferior performance arising from the limitations discussed in Section 1.

The resulting deviance values for the aforementioned methods in comparison to EGO-MDA are presented in Figure~\ref{fig:resulting_bands} (a).
The spectral-bands from RF-MDA and R-MDA, used as respective initialization for NM-MDA and EGO-MDA have relatively higher deviance values. Optimization of initial spectral-bands using EGO and NM resulted in 99.28\% and 83.58\% improvement over R-MDA and RF-MDA, respectively in the median values of deviance.
Both search procedures --- EGO and NM --- improve the deviance values, suggesting that both strategies of random initializations of bands and the important uniform-sized bands can be improved by updating their widths and locations. The large spread observed in deviance of R-MDA is significantly reduced using EGO-MDA on these initialized bands.
 
Larger values of deviance for R-MDA are observed because spectral-bands are determined without prior knowledge unlike in RF-MDA, where important spectral-bands are chosen based on the GINI index values. NM-MDA achieves improvements over RF-MDA, because physical processes may not necessarily manifest themselves in bands of uniform size as assumed in RF-MDA. Therefore, NM-MDA can adjust the band sizes further to improve the deviance. EGO-MDA performs the best because EGO performs relatively better in large dimensions, in contrast to NM search which deteriorates with increasing dimensions. 

Ultimately, resulting ten spectral-bands from EGO-MDA achieve the least deviance with a 77.27\% lower median deviance than NM-MDA.
These EGO-MDA spectral-bands when used for classification have 91.6\% accuracy compared to 88.7\% accuracy for \textit{RF-MDA}.
The width and the location of the optimal spectral-bands, as identified by EGO-MDA are presented in Figure~\ref{fig:resulting_bands} (b). 
Among the ten bands, $\alpha_2$ and $\alpha_7$ indicate opposite trends when comparing among two classes --- class $B$ is characterized by energy reduction in $\alpha_2$ and energy increase in $\alpha_7$. 
Furthermore, these spectral-bands indicate changes in energy over time that can be very difficult to identify from the time series signal alone.
The knowledge about spectral-band showing change can be used to further investigate the physical phenomena causing changes. 
Based on the behavior of these identified spectral-bands, process monitoring systems can be developed using these spectral-bands as channels to detect sharp change points \parencite{iquebal2020change} and incipient changes \parencite{wang2018dirichlet} occurring during the manufacturing process.

\begin{figure}
\begin{center}
\includegraphics[width = 6 in,height = 1.5 in]{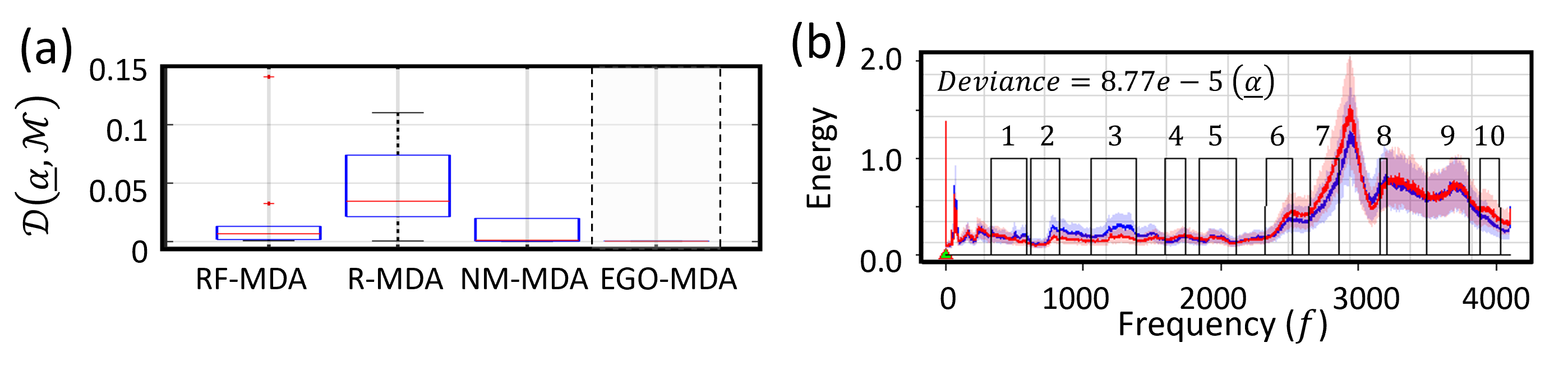}
\end{center}
\vspace{-0.1 in}
\caption{EGO-MDA performance assessment (a) Comparison against competing methods. (b) Final spectral-bands identified by EGO-MDA.}
\label{fig:resulting_bands}
\end{figure}

\section{Conclusions and Future Work}

\noindent In this work, we provide an approach to address an important, yet mostly neglected problem of spectral-banding. 
The problem of spectral-banding is a novel learning task.
While many natural and engineering processes are commonly compared based on differences in the energy of certain spectral-bands, there does not exist an unsupervised approach to identify the spectral-bands that optimally capture these differences or changes.
EGO-MDA introduced in this paper addresses the identification of optimal spectral-bands to discriminate between two classes of spectra, such as those from a normal versus an anomalous process, by reducing the false negatives (not selecting the bands that capture the energy differences in the spectra) and false positives (selecting the bands whose energy content does not change) and reducing the number of channels required for monitoring the process. These channels can also lead to the discovery of novel phenomena in engineering processes because frequency channels can directly indicate the physical attributes of a phenomenon taking place.  EGO-MDA can be useful in several domains of engineering applications that involve spectral processing, including the monitoring of mechanical systems, spectroscopic studies to identify chemical compounds, and acoustic signal processing for speech recognition.

We underscore the performance limitations of EGO-MDA in pathological cases such as low separation between centroids ($\rho$) of the mixing components and underspecified models. A remedy to such pathological cases can be obtained by appropriately choosing the sample size and the number of iterations of the algorithm. In a synthetic data set, EGO-MDA accurately identifies spectral-bands within a 5\% error of maximum band size. In a real-world application case study, EGO-MDA outperforms other methods such as bands of uniform size obtained from the Random Forest classifier, randomized bands, and Nelder-Mead simplex search. 

An important parameter of EGO-MDA for spectral-banding is the number of spectral-bands $L$. The choice of $L$ is specified $\textit{a priori}$ and it can be challenging to identify the optimal number of spectral-bands. This challenge remains to be addressed.
Two possible approaches can be employed to address this challenge. In the first approach, during initialization, $L$ can be large and we allow the merging of spectral-bands during the iterations of EGO-MDA. This leads to final spectral-banding comprising a smaller number of larger spectral-bands that are the optimal number of spectral-bands required for discriminating. However, this approach can reduce the dimension of feature space during the search process leading to additional sophistication in the search procedure. 
Another alternative is to adopt a non-parametric Bayesian approach by introducing a prior distribution over the choice of $L$. In this approach, the appropriate choice of $L$ can be determined using stochastic processes for modeling the prior.

The application can be further extended to image segmentation, specifically involving images from different classes where fixed regions show changes.
In this scenario, criteria for image segmentation is the changes taking place between the two image classes within the fixed regions.
Continuous pixel boundaries --- much like spectral-band boundaries ---  around the region indicating change correspond to the segmented image.

\vspace{-0.25 in}
\section*{Acknowledgments}

The authors thank Suhas Bhandarkar and his team at LLNL for providing useful information about the shell polishing process during the development of the method. The authors would like to especially thank Chantel Aracne-Ruddle and Sean Michael Hayes for supplying the vibration data sets from the polishing process at LLNL.

\vspace{-0.25 in}
\section*{Funding}
The authors acknowledge the generous support from Department of Energy grant no. DE-AC52-07NA27344 through Subcontract no. B646055 and B640459 via the Lawrence Livermore National Laboratory, National Science Foundation under grant no. IIS-1849085, and Texas A\&M University X-grants.



\vspace{-0.25 in}
\printbibliography

\end{document}